\documentclass[letterpaper, 10 pt, conference]{ieeeconf}  

\IEEEoverridecommandlockouts                              

\usepackage{graphics} 
\usepackage{epsfig} 
\usepackage{times} 
\usepackage{amsmath} 
\usepackage{amssymb}  
\usepackage{authblk}
\usepackage{bm}
\usepackage{algorithm}
\usepackage{algorithmic}
\usepackage{multicol}
\usepackage{booktabs}
\usepackage{graphicx}
\usepackage{subcaption}

\usepackage{amsthm}
\usepackage{makecell}
\newtheorem{proposition}{Proposition}
\usepackage{multirow}
\usepackage{xcolor}
\usepackage{threeparttable}
\usepackage{makecell}

\usepackage{cite}

\makeatletter
\let\NAT@parse\undefined
\makeatother
\usepackage[hidelinks]{hyperref}
\hypersetup{
    hidelinks=false,
    colorlinks=false,
    pdfborder={0 0 1},
    linkbordercolor={1 0 0},
    citebordercolor={1 0 0},
    urlbordercolor={1 0 0}
}

\title{\LARGE \bf
Large-Scale Continuous Occupancy Mapping via Variance-Weighted Submap Joining}

\author{Zhuhua Bai*, Yingyu Wang, Liang Zhao, Felix Kong, Shoudong Huang}

\begin{document}

\maketitle
\thispagestyle{empty}
\pagestyle{empty}

\begin{abstract}

Continuous occupancy mapping methods that learn spatial representations
from sensor observations support predictions at arbitrary locations,
but training a single global model can become computationally expensive
as the amount of training data grows during large-scale mapping. Constructing local
continuous occupancy submaps reduces the size of individual
learning problems, but relative pose errors can cause inconsistent
predictions in overlapping regions.
We address these challenges by formulating
large-scale continuous occupancy mapping as a submap joining problem.
Each local submap is represented by a relevance vector machine
regressing a continuous average log-odds field. Repeated observations
within each cell are summarized by their average log-odds and
observation count, yielding a Gaussian likelihood with count-dependent
precision and enabling Gaussian-conjugate inference. The resulting
local fields provide predictive means, variances, and analytical
spatial derivatives. We then formulate submap joining as a joint
pose--field optimization problem, with consistency constraints
weighted by the predictive uncertainty of overlapping local submaps.
With uncertainty weights held fixed within each Gauss--Newton iteration, our
joint pose--field system can be reduced to an
equivalent pose-only system.
After the optimized submap poses are obtained, the global field value at
any covered location can be recovered in closed form through
variance-weighted fusion of predictions from the corresponding local submaps, without training an additional global model.
Experiments on simulated and real-world 2D LiDAR datasets
demonstrate improved pose and occupancy-map accuracy over
the evaluated baselines, reduced map storage and joining
time, and improved map-construction scalability.
\end{abstract}

\section{INTRODUCTION}
Continuous occupancy mapping methods,
such as Gaussian process (GP) approaches~\cite{o2012gaussian},
Hilbert Maps~\cite{ramos2016hilbert}, and relevance vector
machine (RVM) approaches~\cite{duong2022autonomous},
learn continuous spatial representations from sensor observations,
allowing occupancy to be predicted at arbitrary locations. However, the growing amount of training data required to map large
environments can make constructing a single global model
computationally expensive. A practical strategy is therefore to
partition the observations into spatially bounded local submaps. This,
however, raises another important problem: how to combine overlapping
continuous occupancy submaps into a consistent global map, especially
when the estimated relative poses between submaps contain errors. These challenges motivate us to address large-scale continuous occupancy mapping via submap joining.

Existing occupancy-based submap joining methods rely on discrete grid
representations~\cite{11029136,wang2024grid}. Although effective,
evaluating local occupancy values at query locations transformed into
each submap frame requires interpolation, while the gradients required
for optimization are approximated using finite differences. Moreover,
the uncertainty associated with occupancy predictions from individual
submaps is usually not explicitly incorporated into the
optimization. In contrast, Bayesian continuous occupancy maps can provide analytical
derivatives and predictive uncertainty. However, existing continuous
occupancy mapping methods have mainly focused on map construction~\cite{o2012gaussian,ramos2016hilbert,duong2022autonomous},
map fusion with given relative poses~\cite{doherty2016probabilistic,
zhi2019continuous}, or registration against a fixed field~\cite{wu2023log}. Joining multiple continuous
occupancy submaps while correcting their uncertain relative poses
remains less explored.

Among these approaches, GP methods provide
predictive distributions, but standard dense GP inference involves
matrix factorizations whose cost grows cubically with the number of
observations, limiting scalability to large
environments~\cite{o2012gaussian,jadidi2017warped}. Sparse GP methods and Hilbert Maps improve scalability through sparse approximations or finite-dimensional features, but usually require inducing points or feature configurations to be specified or optimized~\cite{ramos2016hilbert,senanayake2017bayesian}. RVM-based occupancy mapping provides another sparse alternative, where a data-dependent representation is learned through automatic relevance determination~\cite{tipping1999relevance,tipping2001sparse,duong2022autonomous,tipping2003fast}. This makes RVM a convenient choice for constructing continuous occupancy submaps.

Existing RVM-based occupancy mapping is commonly formulated as binary
classification with a non-Gaussian likelihood, which requires approximate
posterior inference and thereby introduces additional computational cost~\cite{duong2022autonomous,tipping2003fast}. We instead perform regression on inverse-sensor log-odds observations~\cite{hornung2013octomap,thrun2002probabilistic} and summarize repeated observations in each cell by their average log-odds and observation count. This reduces the amount of training data while preserving the observation strength through the likelihood precision. With a Gaussian likelihood, the weight posterior can be obtained analytically through conjugate inference, avoiding approximate posterior inference and facilitating efficient construction of local occupancy submaps.

To address inconsistencies in the overlapping regions of the constructed local submaps, we formulate continuous submap joining as a variance-weighted joint pose--field optimization problem. At common global locations, predictions from different submaps are evaluated in
their local frames using the estimated poses and constrained to agree
with a common global field. The consistency constraints are
weighted by the predictive uncertainty of the overlapping local submaps. We further show that, with uncertainty weights held fixed within each Gauss--Newton (GN)  iteration, the joint GN equations admit an equivalent
reduced system involving only the submap poses, preserving the
pose--field coupling of the original formulation. After the optimized submap poses are obtained, the continuous global field is recovered in closed form through variance-weighted fusion of local predictions, without training an additional global model.
The overall pipeline is illustrated in Fig.~\ref{fig:framework}.
\begin{figure*}[t]
    \vspace*{4pt}
    \centering
    \includegraphics[
        width=\textwidth,
        trim=0cm 0cm 0cm 0cm,
        clip
    ]{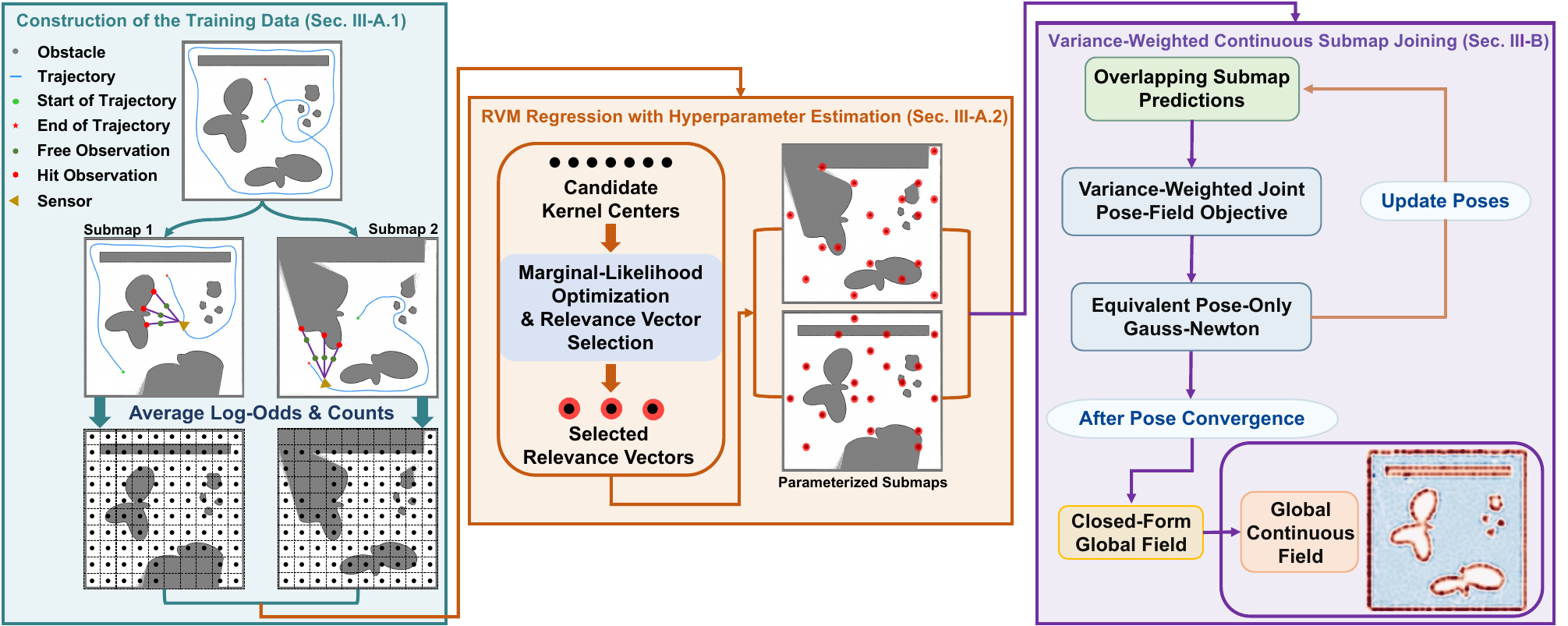}
    \caption{Overview of the proposed framework.
(1) Construction of the Training Data (Sec.~\ref{h}):
local observations are summarized by cell-wise average log-odds
and observation counts.
(2) RVM Submap Learning and Prediction
(Sec.~\ref{A-2} and Sec.~\ref{A3} ): Gaussian RVM regression constructs continuous submaps
and provides predictive means and variances.
(3) Variance-Weighted Continuous Submap Joining (Sec.~\ref{sec:Average Log-odds RVM for Submap Joining}):
overlapping submaps are coupled through a variance-weighted joint pose–field objective. With uncertainty weights held fixed within each GN iteration, the joint GN system reduces to an equivalent pose-only GN system.  After
the optimized submap poses are obtained, the continuous global field is recovered in closed form.}
\label{fig:framework}
\vspace*{-2.5pt}
\end{figure*}

The main contributions of this work are summarized as follows:

\begin{enumerate}
    \item We formulate large-scale continuous occupancy mapping as a
continuous submap joining problem. The continuous local submaps
provide analytical derivatives and predictive uncertainty, which are
exploited in a variance-weighted joint pose--field optimization to refine relative submap poses.

    \item  We develop an RVM-based continuous occupancy mapping method for
submap construction, where repeated observations in each cell are
summarized by their average log-odds and observation count to enable
Gaussian-conjugate inference. The proposed method further improves
the efficiency of continuous submap construction by reducing the
training data and avoiding approximate posterior inference.

    \item We show that, with uncertainty weights held fixed within each GN iteration, the pose increments of the joint GN system are independent of the current global field values. This property enables an equivalent reduced pose-only GN solver, while the global field value at any covered location can be recovered in closed form after the optimized submap poses are obtained.

    \item Experiments demonstrate higher pose and map accuracy and more efficient submap joining than the evaluated baselines, along with improved map-construction scalability.
\end{enumerate}

\section{RELATED WORK}
\subsection{Continuous Occupancy Mapping and Fusion}
A variety of continuous occupancy representations have
been developed.
GP-based methods~\cite{o2012gaussian,jadidi2017warped} provide
continuous occupancy estimates but become computationally
expensive as observations accumulate. Hilbert Maps~\cite{ramos2016hilbert}
and Bayesian Hilbert Maps~\cite{senanayake2017bayesian} improve
scalability using finite-dimensional features, while RVM-based
methods~\cite{tipping1999relevance,tipping2001sparse,tipping2003fast}
learn sparse, data-dependent representations through automatic
relevance determination. SBKM~\cite{duong2022autonomous} uses
sparse Bayesian kernels for incremental occupancy mapping,
and GMMap~\cite{li2024gmmap} uses Gaussian mixtures for efficient
continuous mapping. These methods mainly focus on continuous
occupancy map construction.

Continuous occupancy map-fusion
methods~\cite{doherty2016probabilistic,zhi2019continuous,nabarro2026distributed}
combine local maps in prediction or parameter space, generally
assuming given relative poses. They therefore do not correct
relative pose errors during fusion, potentially yielding
inconsistent occupancy estimates when submaps are inaccurately
registered.
\subsection{Submap Joining and Pose--Map Coupling}
Classical submap joining methods estimate relative submap poses
to improve global consistency~\cite{huang2008sparse,zhao2013linear}.
Occupancy-SLAM~\cite{11029136} jointly optimizes robot poses and
global occupancy values, while Grid-based Submap Joining
(GBSJ)~\cite{wang2024grid} extends this formulation to local
submap frames. However, GBSJ requires grid interpolation for local occupancy queries and does not explicitly account for local occupancy uncertainty in its joining objective.

Continuous field representations have also been used for map registration.
Log-GPIS~\cite{wu2023log} registers incoming measurements against
a continuous distance field using its analytical gradients and
uncertainty. The field remains fixed during pose estimation
and is subsequently updated through measurement fusion,
rather than jointly optimized with the poses.

In contrast, our method jointly optimizes the poses of
continuous occupancy submaps and global field values
at selected query locations, weighting consistency
constraints by the predictive uncertainty of overlapping
submaps.
\section{METHOD}
\label{sec:conclusion}
This section presents the proposed framework in two stages.
We first construct continuous occupancy submaps using Gaussian-conjugate
RVM regression on cell-wise average log-odds observations, and then formulate
variance-weighted continuous submap joining using their predictive mean,
variance, and analytical spatial derivatives.

\subsection{Continuous Occupancy Submap Construction}
\label{A}
\subsubsection{Construction of the Training Data}\label{h}
Each local submap is constructed in its own coordinate frame.
LiDAR scans are transformed from the sensor frame into the
local submap frame using the corresponding scan poses.
The transformed rays are then traced on a multi-resolution
grid, with finer resolution near obstacle boundaries.
For each ray, the endpoint cell receives one hit observation
with occupancy probability $p_{\mathrm{hit}}$, while each
cell traversed before the endpoint cell receives one free
observation with occupancy probability $p_{\mathrm{free}}$,
according to the inverse sensor model~\cite{hornung2013octomap,thrun2002probabilistic}.

Let $\mathbf{x}_i$ denote the center of the $i$-th observed
cell in the local submap frame, and let $n_i$ be the total
number of observations assigned to this cell across all rays
in the submap. For its $j$-th observation, $j=1,\ldots,n_i$,
let $p_{ij}$ denote the occupancy probability assigned by
the inverse sensor model. The corresponding log-odds value is
\begin{equation}
z_{ij}=\ln\frac{p_{ij}}{1-p_{ij}},
\qquad
p_{ij}=
\begin{cases}
p_{\mathrm{hit}}, & \text{hit observation},\\
p_{\mathrm{free}}, & \text{free observation}.
\end{cases}
\label{eq:logodds}
\end{equation}
We use $p_{\mathrm{hit}}=0.7$ and $p_{\mathrm{free}}=0.4$ as in~\cite{hornung2013octomap,thrun2002probabilistic}.

For the $i$-th observed cell, the average log-odds is
$\bar{z}_i=\frac{1}{n_i}\sum_{j=1}^{n_i}z_{ij}$.
The resulting training dataset is
$\mathcal{D}=\{(\mathbf{x}_i,\bar{z}_i,n_i)\}_{i=1}^{N_c}$,
where $N_c$ is the number of observed cells.

The average log-odds $\bar{z}_i$ represents the local occupancy tendency,
while the observation count $n_i$ reflects the amount of observational
support and controls the likelihood precision in Eq.~\eqref{eq3}. Unlike the accumulated log-odds used in conventional occupancy
mapping~\cite{hornung2013octomap}, the average does not grow in magnitude merely
through repeated identical observations. This reduces
count-induced differences in the training targets, facilitating
comparisons between the learned local average log-odds fields
during submap joining. 
\subsubsection{RVM Regression with Hyperparameter Estimation}
\label{A-2}
For an arbitrary location $\mathbf{x}$ in the local submap frame, let 
$\mathbf{x}_m$ denote the $m$-th candidate kernel center, initialized 
from the observed cell centers. We use the radial basis function (RBF) kernel
$k(\mathbf{x},\mathbf{x}_m)
=\eta\exp(-\gamma\|\mathbf{x}-\mathbf{x}_m\|^2)$,
where $\eta$ is the kernel amplitude and $\gamma$ controls its spatial
scale~\cite{duong2022autonomous}. The feature vector is
$\boldsymbol{\phi}(\mathbf{x})
=[k(\mathbf{x},\mathbf{x}_1),\ldots,
k(\mathbf{x},\mathbf{x}_{N_c})]^\top$.

The continuous occupancy function in average log-odds form is represented as
\begin{equation}
F(\mathbf{x})
=
\boldsymbol{\phi}(\mathbf{x})^\top\bm{w}+b,
\label{eq:rvm_function}
\end{equation}
where $\bm{w}\in\mathbb{R}^{N_c}$ is the weight vector and $b$ is
the bias.

We model each log-odds observation assigned to cell $i$ as
$z_{ij}=F(\mathbf{x}_i)+\varepsilon_{ij}$, where
$\varepsilon_{ij}\sim\mathcal{N}(0,\beta^{-1})$ are
conditionally independent Gaussian noise terms. Under this Gaussian observation
model, the cell-wise average satisfies
\begin{equation}
\label{eq3}
\bar{z}_i\sim\mathcal{N}\!\left(F(\mathbf{x}_i),(\beta n_i)^{-1}\right).
\end{equation}
Therefore, stacking all observed cells yields the data likelihood
\begin{equation}
p(\bar{\mathbf{z}}\mid\bm{w},\beta,\mathbf{N})
=
\mathcal{N}
\!\left(
\bar{\mathbf{z}}
\,\middle|\,
\boldsymbol{\Phi}\bm{w}+b\mathbf{1},
\beta^{-1}\mathbf{N}^{-1}
\right),
\label{eq:avg_logodds_likelihood}
\end{equation}
where $\bar{\mathbf{z}}=[\bar{z}_1,\ldots,\bar{z}_{N_c}]^\top$,
$\mathbf{N}=\mathrm{diag}(n_1,\ldots,n_{N_c})$, and the $i$-th row of
$\boldsymbol{\Phi}$ is $\boldsymbol{\phi}(\mathbf{x}_i)^\top$.

Following the standard RVM formulation~\cite{tipping2001sparse},
the $m$-th weight $w_m$ has an independent zero-mean Gaussian prior
with precision $\alpha_m$.
The weight prior is
\begin{equation}
p(\bm{w}\mid\boldsymbol{\alpha})
=
\mathcal{N}(\bm{w}\mid\mathbf{0},\mathbf{A}^{-1}),
\label{eq:rvm_prior}
\end{equation}
where $\boldsymbol{\alpha}
= [\alpha_1,\ldots,\alpha_{N_c}]^\top$
and $\mathbf{A}=\operatorname{diag}(\boldsymbol{\alpha})$.

Combining the Gaussian data likelihood in Eq.~\eqref{eq:avg_logodds_likelihood}
with the prior in Eq.~\eqref{eq:rvm_prior} yields the exact weight posterior
\begin{equation}
p(\bm{w} \mid \mathcal{D}, \bm{\alpha}, \beta)
=
\mathcal{N}(\bm{w} \mid \bm{\mu}, \bm{\Sigma}),
\label{eq:rvm_posterior}
\end{equation}
where the posterior mean $\bm{\mu}$ and covariance $\bm{\Sigma}$ are given by
\begin{equation}
\bm{\Sigma}
=
\left(
\beta \bm{\Phi}^{\top}\mathbf{N}\bm{\Phi}
+\mathbf{A}
\right)^{-1},
\quad
\bm{\mu}
=
\beta\bm{\Sigma}\bm{\Phi}^{\top}\mathbf{N}
(\bar{\mathbf{z}}-b\mathbf{1}).
\label{eq:rvm_posterior_parameters}
\end{equation}

\textbf{Hyperparameter Estimation.}
The hyperparameters $\boldsymbol{\alpha}$ and $\beta$ are estimated by
type-II maximum likelihood by maximizing the log marginal likelihood
\begin{equation}
\mathcal{L}(\boldsymbol{\alpha},\beta)
=
-\frac{1}{2}
\left[
N_c\log(2\pi)
+\log|\mathbf{C}|
+\widetilde{\mathbf{z}}^\top
\mathbf{C}^{-1}
\widetilde{\mathbf{z}}
\right],
\label{eq:marginal_likelihood}
\end{equation}
where
$\mathbf{C}
=\beta^{-1}\mathbf{N}^{-1}
+\boldsymbol{\Phi}\mathbf{A}^{-1}\boldsymbol{\Phi}^{\top}$
and
$\widetilde{\mathbf{z}}=\bar{\mathbf{z}}-b\mathbf{1}$.

We use the sequential sparse Bayesian learning scheme~\cite{tipping2003fast}
to update the active basis functions and their corresponding $\alpha_m$.
The noise precision $\beta$ is re-estimated in the same learning loop as
\begin{equation}
\beta^{\mathrm{new}}
=
\frac{
N_c-\sum_{m=1}^{N_{\mathrm{RV}}}\gamma_m
}{
\sum_{i=1}^{N_c} n_i e_i^2
},
\qquad
\gamma_m=1-\alpha_m\sigma_{mm}^2,
\label{eq:beta_update}
\end{equation}
where $N_{\mathrm{RV}}$ is the number of active relevance vectors,
$\sigma_{mm}^2=[\boldsymbol{\Sigma}]_{mm}$ is the posterior
variance of $w_m$, and $e_i$ is the
$i$-th element of the residual vector
$\mathbf{e}=\bar{\mathbf{z}}-b\mathbf{1}-\boldsymbol{\Phi}\boldsymbol{\mu}$.

\subsubsection{Continuous Occupancy Prediction}
\label{A3}
After the model has been trained, for an arbitrary query location $\mathbf{x}_*$ in the local submap
frame, the predictive distribution is
\begin{equation}
\label{eq:pred}
\begin{gathered}
p(F(\mathbf{x}_*)\mid\mathbf{x}_*) = \mathcal{N}(F(\mathbf{x}_*)\mid\mu_*,\sigma_*^2) \\[4pt]
\mu_* = \boldsymbol{\phi}(\mathbf{x}_*)^\top\boldsymbol{\mu}+b,\quad \sigma_*^2 = \boldsymbol{\phi}(\mathbf{x}_*)^\top\boldsymbol{\Sigma}\,\boldsymbol{\phi}(\mathbf{x}_*).
\end{gathered}
\end{equation}

Since $\mu_*$ is expressed in average log-odds, its point estimate can be mapped directly to occupancy probability through the inverse-logit function~\cite{thrun2002probabilistic}. To additionally account for predictive uncertainty, we use the standard logistic-Gaussian approximation~\cite{seeger2004gaussian}:
\begin{equation}
\label{Continuous Occupancy Prediction}
\overline{P}_{\text{occ}}(\mathbf{x}_*) \approx \frac{1}{1+\exp(-\kappa\mu_*)},\quad \kappa = \frac{1}{\sqrt{1+\pi\sigma_*^2/8}}.
\end{equation}
\subsection{Variance-Weighted Continuous Submap Joining}
\label{sec:Average Log-odds RVM for Submap Joining}

Each trained RVM submap $i$ provides a continuous predictive
average log-odds mean $\mu_i(\mathbf{x})$ and variance
$\sigma_i^2(\mathbf{x})$, together with analytical spatial derivatives from the differentiable RBF kernel. These quantities enable joint optimization of submap poses
and global field values by enforcing consistency with
overlapping continuous local fields.

\subsubsection{Variance-Weighted Joint Pose--Field Formulation}

Consider an arbitrary location $\mathbf{x}_g$ in the global coordinate
frame covered by at least
one submap. Its coordinate in the local frame of submap $i$ is
$\mathbf{p}_{i\mathbf{x}_g}
=\mathbf{R}_i^\top(\mathbf{x}_g-\mathbf{t}_i)$,
where $\mathbf{R}_i$ and $\mathbf{t}_i$ define the transformation from
the local frame of submap $i$ to the global frame.
Let $M(\mathbf{x}_g)$ denote the continuous global average log-odds
field at $\mathbf{x}_g$. Its consistency with the predictions of the
overlapping submaps is measured by
\begin{equation}
E(\mathbf{x}_g)
=
\sum_{i\in\mathcal{O}(\mathbf{x}_g)}
\frac{
\left[
M(\mathbf{x}_g)-\mu_i(\mathbf{p}_{i\mathbf{x}_g})
\right]^2
}{
\sigma_i^2(\mathbf{p}_{i\mathbf{x}_g})+\epsilon
},
\label{eq:pointwise_consensus}
\end{equation}
where $\mathcal{O}(\mathbf{x}_g)$ denotes the set of submaps whose
valid prediction regions cover $\mathbf{x}_g$, and $\epsilon>0$ is a
small constant for numerical stability.

To formulate an optimization problem with a finite number
of objective terms, we select global query locations
$\{\mathbf{m}_j\}_{j=1}^{G}$ and introduce an auxiliary variable
$M_j=M(\mathbf{m}_j)$ at each location $\mathbf{m}_j$. Only query locations covered by at least one valid submap
are retained. These variables represent
the common global field values against which overlapping
submap predictions are compared; they do not parameterize
the continuous global field.

The first submap defines the global reference frame. Let
$\mathbf{T}^r$ stack the poses $(\mathbf{R}_i,\mathbf{t}_i)$ of the
remaining submaps, and let
$\mathbf{M}=[M_1,\ldots,M_G]^\top$ collect all the auxiliary variables.
Based on Eq.~\eqref{eq:pointwise_consensus}, the joint continuous submap joining problem is formulated as
\begin{equation}
\label{eq:joint_objective}
\min_{\mathbf{T}^{r},\mathbf{M}}
f(\mathbf{T}^{r},\mathbf{M})
=
\sum_{j=1}^{G}
\sum_{i\in\mathcal{O}_j}
w_{ij}
\left[
M_j-\mu_i(\mathbf{p}_{i\mathbf{m}_j})
\right]^2,
\end{equation}
where $w_{ij}=
[\sigma_i^2(\mathbf{p}_{i\mathbf{m}_j})+\epsilon]^{-1}$,
$\mathbf{p}_{i\mathbf{m}_j}
=\mathbf{R}_i^\top(\mathbf{m}_j-\mathbf{t}_i)$,
and $\mathcal{O}_j$ contains the submaps whose valid prediction regions
cover $\mathbf{m}_j$.
\subsubsection{Equivalent Pose-Only GN Optimization}
We now show that, with the uncertainty weights held fixed
within each GN iteration, the optimization problem \eqref{eq:joint_objective} admits an equivalent pose-only GN solution that produces the same pose increment as the full joint pose--field GN system.
At iteration $k$, the
predictive variances are evaluated at the current submap poses
and the resulting weights $w_{ij}^{(k)}$ are held fixed when forming and solving
the GN system. They are recomputed after each pose
update. 

At iteration $k$, each residual is
\begin{equation}
\label{eq:residual}
r_{ij}
=
\sqrt{w_{ij}^{(k)}}
\left[
M_j-\mu_i(\mathbf{p}_{i\mathbf{m}_j})
\right].
\end{equation}

Let $\mathbf{r}$ denote the stacked residual vector at the current
iteration, and let $\mathbf{J}_T$ and $\mathbf{J}_M$ denote its
Jacobians with respect to the submap poses and the auxiliary variables
$\mathbf{M}$, respectively. The GN normal equations for the joint problem are
\begin{equation}
\begin{bmatrix}
\mathbf{U} & \mathbf{W}\\
\mathbf{W}^{\top} & \mathbf{V}
\end{bmatrix}
\begin{bmatrix}
\Delta\mathbf{T}\\
\Delta\mathbf{M}
\end{bmatrix}
=
\begin{bmatrix}
\mathbf{b}_T\\
\mathbf{b}_M
\end{bmatrix},
\label{eq:joint_GN}
\end{equation}
where
$\mathbf{U}=\mathbf{J}_T^\top\mathbf{J}_T$,
$\mathbf{V}=\mathbf{J}_M^\top\mathbf{J}_M$,
$\mathbf{W}=\mathbf{J}_T^\top\mathbf{J}_M$,
$\mathbf{b}_T=-\mathbf{J}_T^\top\mathbf{r}$, and
$\mathbf{b}_M=-\mathbf{J}_M^\top\mathbf{r}$.

Since each residual involves only one auxiliary variable,
$\mathbf{V}$ is diagonal, with
\begin{equation}
\label{eq:hmm}
[\mathbf{V}]_{jj}
=
\sum_{i\in\mathcal{O}_j}w_{ij}^{(k)}
>0.
\end{equation}
Thus, $\mathbf{V}$ is positive definite and can be efficiently inverted.
This allows $\Delta\mathbf{M}$ to be eliminated from the joint
GN normal equations. The following proposition further shows
that the resulting pose increment is independent of the current
auxiliary variables
$\mathbf{M}$.

\begin{proposition}
\label{prop:pose_only}
For each GN iteration with fixed uncertainty weights,
the pose increment $\Delta\mathbf{T}$ obtained from the full
joint pose--field system is independent of the current auxiliary variables $\mathbf{M}$. Consequently, the full joint
GN pose update has an equivalent pose-only
solution.
\end{proposition}
\begin{proof}
For each GN iteration, the residual vector has the
form
\begin{equation}
\label{eq:residual_matrix}
\mathbf{r}
=
\mathbf{J}_{M}\mathbf{M}
-
\mathbf{h}(\mathbf{T}^{r}),
\end{equation}
where $\mathbf{h}(\mathbf{T}^{r})$ contains the weighted local
RVM predictions.

Eliminating $\Delta\mathbf{M}$ from
Eq.~\eqref{eq:joint_GN} via the Schur complement \cite{Zhang2005Schur} gives
\begin{equation}
\label{eq:reduced_pose_system}
\left(
\mathbf{U}
-
\mathbf{W}\mathbf{V}^{-1}\mathbf{W}^{\top}
\right)
\Delta\mathbf{T}
=
\mathbf{b}_{T}
-
\mathbf{W}\mathbf{V}^{-1}\mathbf{b}_{M}.
\end{equation}

Substituting Eq.~\eqref{eq:residual_matrix} into the definitions of
$\mathbf{b}_T$ and $\mathbf{b}_M$, the right-hand side
of Eq.~\eqref{eq:reduced_pose_system} becomes
\begin{equation}
\label{eq:reduced_rhs}
\mathbf{J}_{T}^{\top}\mathbf{h}
-
\mathbf{W}\mathbf{V}^{-1}
\mathbf{J}_{M}^{\top}\mathbf{h},
\end{equation}
where all terms containing $\mathbf{M}$ cancel. With the uncertainty weights fixed, $\mathbf{J}_T$,
$\mathbf{J}_M$, and $\mathbf{h}$ are independent of
$\mathbf{M}$. Since the
coefficient matrix in Eq.~\eqref{eq:reduced_pose_system} is
also independent of $\mathbf{M}$, the resulting pose increment
$\Delta\mathbf{T}$ is independent of the current auxiliary variables
$\mathbf{M}$.
\end{proof}

Therefore, with uncertainty weights held fixed within each GN iteration, the reduced pose-only system produces exactly the same pose increment as the full joint pose--field GN system. Although only pose increments are solved explicitly, global field consistency remains incorporated into pose estimation through the Schur complement reduction. Unlike conventional pose-graph optimization based on relative-pose constraints, this reduced system retains the pose--field coupling of the joint formulation. Since the pose increment is independent of the current auxiliary variable values, these values need not be explicitly updated or alternated with the pose updates.
\subsubsection{Continuous Global Field}
After the optimized submap poses are
obtained, we can recover the global
average log-odds field from the local RVM submaps in closed form. The pointwise
consistency objective in Eq.~\eqref{eq:pointwise_consensus} is strictly convex with
respect to $M(\mathbf{x}_g)$ because all uncertainty weights
are positive. Therefore, its unique optimum gives the global average log-odds mean
\begin{equation}
\label{eq:global_mean}
\mu_G(\mathbf{{x}_g})
=
\frac{
\displaystyle
\sum_{i\in\mathcal{O}(\mathbf{{x}_g})}
\frac{
\mu_i(\mathbf{p}_{i\mathbf{{x}_g}})
}{
\sigma_i^2(\mathbf{p}_{i\mathbf{{x}_g}})+\epsilon
}
}{
\displaystyle
\sum_{i\in\mathcal{O}(\mathbf{{x}_g})}
\frac{
1
}{
\sigma_i^2(\mathbf{p}_{i\mathbf{{x}_g}})+\epsilon
}
},
\end{equation}
where $\mathbf{p}_{i\mathbf{{x}_g}}$ is evaluated using the
optimized submap poses.

Assuming conditional independence of the local Gaussian
predictions and treating
$\sigma_i^2(\mathbf{p}_{i\mathbf{{x}_g}})+\epsilon$ as the effective
prediction variance, the corresponding global predictive
variance is
\begin{equation}
\label{eq:global_variance}
\sigma_G^2(\mathbf{x_g})
=
\left[
\sum_{i\in\mathcal{O}(\mathbf{x_g})}
\frac{
1
}{
\sigma_i^2(\mathbf{p}_{i\mathbf{x_g}})+\epsilon
}
\right]^{-1}.
\end{equation}

The global occupancy probability at $\mathbf{x}_g$ is obtained
by substituting $\mu_G(\mathbf{x}_g)$ and $\sigma_G^2(\mathbf{x}_g)$
into Eq.~\eqref{Continuous Occupancy Prediction}.
Consequently, the continuous global occupancy field can be queried
at arbitrary locations without training an additional global model.
\section{Experiments}
We evaluate the proposed framework from three aspects:
(1) submap joining accuracy and the effect of predictive-variance weighting,
(2) the accuracy and efficiency of the proposed continuous occupancy mapping method, and
(3) computational efficiency and scalability.

We evaluate our method against several representative SLAM and
occupancy-based baselines. Cartographer~\cite{hess2016real} and
SLAM Toolbox~\cite{macenski2021slam} are included as representative
pose graph SLAM systems, while GBSJ~\cite{wang2024grid}
serves as a state-of-the-art discrete occupancy-based submap joining
baseline.

To enable a fair comparison with an existing RVM-based continuous occupancy mapping method, we extend SBKM~\cite{duong2022autonomous}, which originally performs continuous occupancy mapping without submap joining, to the same joining framework as our method. Specifically, its latent RVM mean, predictive variance, and analytical kernel derivatives are directly used in its native latent space for submap joining. We denote this extended method as SBKM-Join.

For a fair submap joining comparison, GBSJ, SBKM-Join, and our method
use the same submap partitions and the same internal poses obtained
from Occupancy-SLAM~\cite{11029136}. 

As publicly available real-world 2D LiDAR datasets generally lack
ground-truth, we first evaluate our method quantitatively and
qualitatively in two simulated environments, and then assess its
performance qualitatively on two public real-world datasets.
Table~\ref{tab:datasets} summarizes all datasets used. Depending on the
dataset scale, the number of local submaps ranges from 3 to 10.
\begin{table}[!t]
\vspace*{3.5pt}
\centering
\caption{The parameters of all datasets.}
\label{tab:datasets}
\renewcommand{\arraystretch}{1}

\begin{tabular}{cccc}
\toprule
\textbf{Dataset} &
\textbf{No. Scans} &
\textbf{Duration (s)} &
\textbf{Map Size (m)} \\
\midrule
\textbf{Simulation 1} & 3640  & 117 & $50 \times 50$ \\
\textbf{Simulation 2} & 2680  & 83  & $50 \times 50$ \\
\textbf{Museum b0 1G} & 5522  & 152 & $100 \times 100$ \\
\textbf{Museum b0 EG} & 22650 & 615 & $250 \times 160$ \\
\bottomrule
\end{tabular}
\end{table}

\begin{figure}[!t]
    \centering
    \begin{subfigure}[b]{0.236\textwidth}
        \centering
        \includegraphics[width=\textwidth]{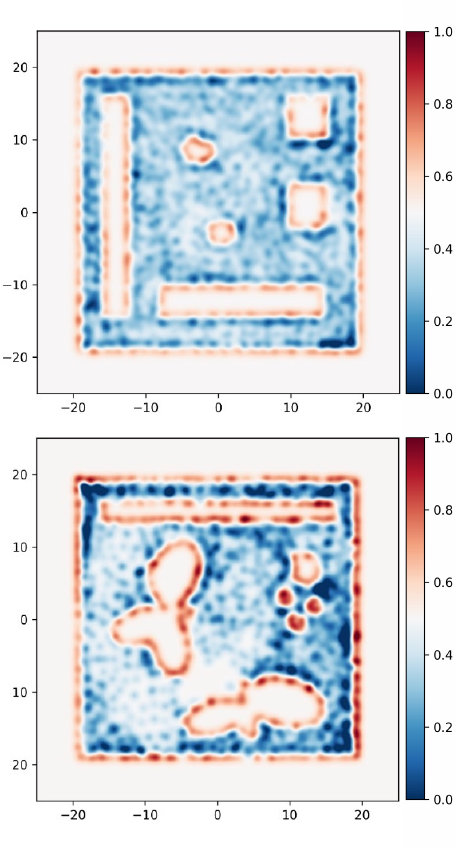}
        \caption{SBKM-Join}
        \label{fig:gt}
    \end{subfigure}
    \hfill
    \begin{subfigure}[b]{0.239\textwidth}
        \centering
        \includegraphics[width=\textwidth]{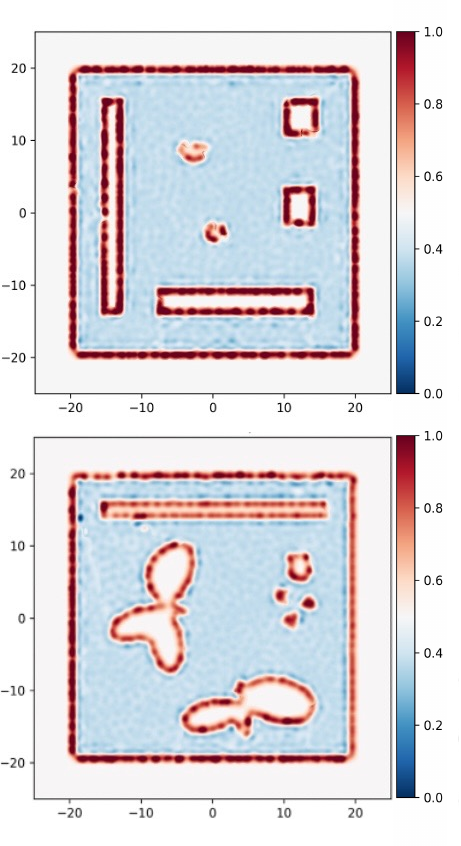}
        \caption{Ours (Full)}
        \label{fig:carto}
    \end{subfigure}
    \caption{
Continuous occupancy probability maps from SBKM-Join and
Ours (Full), illustrating the simulation layouts and
differences in their occupancy probability predictions.
The first and second rows show Simulation~1 and Simulation~2,
respectively. All panels use the same probability color scale.
    }
\label{fig:comparison}
\vspace{-2pt}
\end{figure}
All experiments were conducted on a CPU-only server equipped
with an AMD EPYC 9354P 32-Core Processor. For our method,
we used an RBF kernel with $\eta=1$ and $\gamma=3$,
and set the bias to $b=-0.05$.
For a fair comparison, all methods used a common 0.2\,m grid
for map visualization and quantitative evaluation.
For visualization and accuracy evaluation, cells were classified
as occupied for $p>0.65$, free for $p<0.45$, and unknown
otherwise.

\subsubsection{Simulation Experiments}
We use the two simulation environments from \cite{wang2024grid},
which contain non-convex obstacles and long corridors.
The general layouts of the two simulation environments
are shown in Fig.~\ref{fig:comparison}. To
simulate real-world data acquisition, zero-mean Gaussian noise with a
standard deviation of $0.02$\,m is added to each LiDAR beam, while the
odometry is corrupted with standard deviations of $0.04$\,m in
translation and $0.003$\,rad in rotation. 

Table~\ref{tab:pose_errors} reports the translation and rotation errors
using Mean Absolute Error (MAE) and Root Mean Square Error (RMSE).
Our method, denoted Ours (Full), achieves the lowest pose errors across
all metrics in both simulated environments. Compared with GBSJ, the
results demonstrate the effectiveness of the proposed continuous
submap joining framework.

To isolate the effect of predictive-variance weighting, we use an
ablated version of our method with uniform weights for all local-field
consistency constraints, denoted Ours (Uniform). Ours (Full)
consistently improves all four pose-error metrics over Ours (Uniform),
demonstrating the benefit of incorporating predictive uncertainty into
submap joining.

Ours (Uniform) remains competitive with or better than GBSJ, showing
that the proposed continuous joint pose--field formulation remains
effective even without uncertainty weighting. In contrast, the larger pose errors of SBKM-Join, which uses
the same joining framework but a different local mapping
formulation, support the effectiveness of the proposed
continuous occupancy submap construction method for
accurate submap joining.

\begin{table}[!t]
\vspace*{3.5pt}
\centering
\caption{Comparison of robot pose errors in simulations.}
\label{tab:pose_errors}

\setlength{\tabcolsep}{1.8pt}
\renewcommand{\arraystretch}{1}

\begin{tabular*}{\columnwidth}{@{\extracolsep{\fill}}cccccc@{}}
\toprule
\makecell[c]{\textbf{Dataset}} &
\makecell[c]{\textbf{Method}} &
\makecell[c]{\textbf{MAE}\\\textbf{(T)} $\downarrow$} &
\makecell[c]{\textbf{MAE}\\\textbf{(R)} $\downarrow$} &
\makecell[c]{\textbf{RMSE}\\\textbf{(T)} $\downarrow$} &
\makecell[c]{\textbf{RMSE}\\\textbf{(R)} $\downarrow$} \\
\midrule

\multirow{6}{*}{\textbf{Simulation 1}}
& Cartographer
& 0.2526 & 0.0099 & 0.3066 & 0.0170 \\

& Toolbox
& 0.1465 & 0.0134 & 0.1788 & 0.0192 \\

& GBSJ
& 0.0453 & 0.0012 & 0.0609 & \underline{0.0015} \\

& SBKM-Join
& 0.2837 & 0.0168 & 0.3823 & 0.0174 \\

& Ours (Uniform)
& \underline{0.0282}
& \underline{0.0011}
& \underline{0.0424}
& \underline{0.0015} \\

& Ours (Full)
& \textbf{0.0164}
& \textbf{0.0009}
& \textbf{0.0279}
& \textbf{0.0011} \\

\midrule

\multirow{6}{*}{\textbf{Simulation 2}}
& Cartographer
& 0.1094 & 0.0087 & 0.2041 & 0.0178 \\

& Toolbox
& 0.0493 & 0.0085 & 0.0626 & 0.0127 \\

& GBSJ
& 0.0151 & 0.0011 & 0.0205 & 0.0018 \\

& SBKM-Join
& 0.2532 & 0.0106 & 0.3106 & 0.0158 \\

& Ours (Uniform)
& \underline{0.0078}
& \underline{0.0009}
& \underline{0.0175}
& \underline{0.0017} \\

& Ours (Full)
& \textbf{0.0051}
& \textbf{0.0003}
& \textbf{0.0105}
& \textbf{0.0013} \\

\bottomrule
\end{tabular*}

\vspace{1mm}

\parbox{\columnwidth}{
\footnotesize
\raggedright
Best results are shown in \textbf{bold} and second-best results are
\underline{underlined}. T: Translation (m); R: Rotation (rad).
}
\vspace*{-2.5pt}
\end{table}
We evaluate occupancy maps using the area under the ROC
curve (AUC) and accuracy, following~\cite{wang2024grid,duong2022autonomous}
for ground-truth labeling and unknown-cell handling.
AUC measures occupied/free discrimination across decision
thresholds, while accuracy measures the fraction of correctly
classified cells under the specified evaluation protocol.

To isolate the effect of predictive uncertainty in the final
occupancy prediction, we evaluate Ours w/o pred. var., which
uses the same optimized poses and global predictive mean as
Ours (Full), but sets $\kappa=1$ in Eq.~\eqref{Continuous Occupancy Prediction}. Unlike Ours
(Uniform), this variant retains variance weighting during
submap joining and  only removes the variance correction in
the final occupancy probability conversion. We apply the same
prediction-only ablation to SBKM-Join. We also include Ours w/o count, which ignores
observation counts during training by assigning equal likelihood
precision to all observed cells.

Table~\ref{tab:accuracy} shows that Ours (Full) achieves
the highest AUC and accuracy among all compared methods
on both simulation datasets, outperforming both discrete
and continuous baselines. Removing the predictive-variance correction reduces both
AUC and accuracy for both continuous methods, indicating that predictive uncertainty
benefits not only pose estimation during submap joining, as shown in
Table~\ref{tab:pose_errors}, but also the final occupancy prediction.
The degradation of Ours w/o count further demonstrates the benefit of
retaining observation counts for occupancy-map accuracy.

\begin{table}[!t]
\vspace*{3.5pt}
\centering
\caption{Accuracy of the occupancy maps.}
\label{tab:accuracy}
\renewcommand{\arraystretch}{1}
\setlength{\tabcolsep}{4pt}

\begin{threeparttable}

\begin{tabular}{cccc}
\toprule
\textbf{Dataset} &
\textbf{Method} &
\textbf{AUC} $\uparrow$ &
\textbf{Accuracy} $\uparrow$ \\
\midrule

\multirow{8}{*}{\textbf{Simulation 1}}
& Cartographer & 0.9128 & 0.9215 \\
& SLAM Toolbox & 0.9461 & 0.9339 \\
& GBSJ & \underline{0.9831} & \underline{0.9584} \\
& SBKM-Join w/o pred. var. & 0.8907 & 0.9095 \\
& SBKM-Join & 0.9341 & 0.9367 \\
& Ours w/o count & 0.9256 & 0.9331 \\
& Ours w/o pred. var. & 0.9223 & 0.9277 \\
& Ours (Full) & \textbf{0.9873} & \textbf{0.9623} \\

\midrule

\multirow{8}{*}{\textbf{Simulation 2}}
& Cartographer & 0.9170 & 0.9260 \\
& SLAM Toolbox & 0.9526 & 0.9371 \\
& GBSJ & \underline{0.9954} & \underline{0.9458} \\
& SBKM-Join w/o pred. var. & 0.8812 & 0.9122 \\
& SBKM-Join & 0.9149 & 0.9239 \\
& Ours w/o count & 0.9172 & 0.9263 \\
& Ours w/o pred. var. & 0.9227 & 0.9295 \\
& Ours (Full) & \textbf{0.9975} & \textbf{0.9597} \\

\bottomrule
\end{tabular}

\end{threeparttable}
\vspace*{-2.5pt}
\end{table}
\begin{figure*}[t]
\vspace*{4.5pt}
\centering

\begin{subfigure}[b]{0.1945\textwidth}
    \centering
    \includegraphics[width=\textwidth]{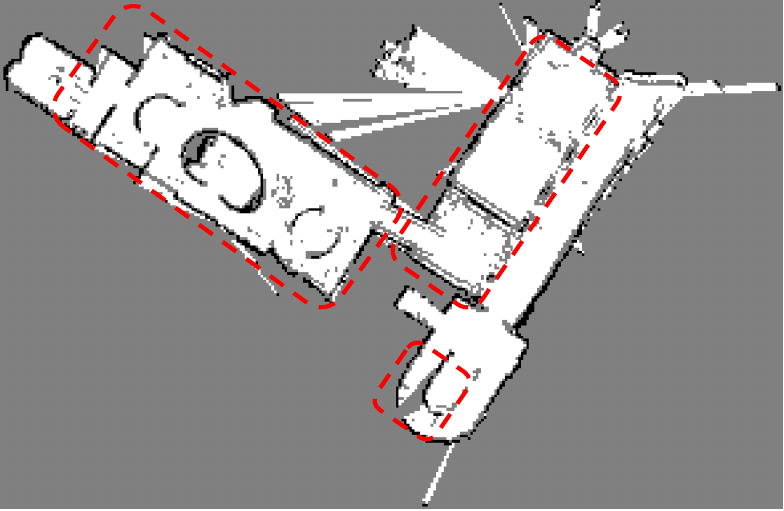}
    \caption{Cartographer}
    \label{fig:1g_carto}
\end{subfigure}
\hfill
\begin{subfigure}[b]{0.195\textwidth}
    \centering
    \includegraphics[width=\textwidth]{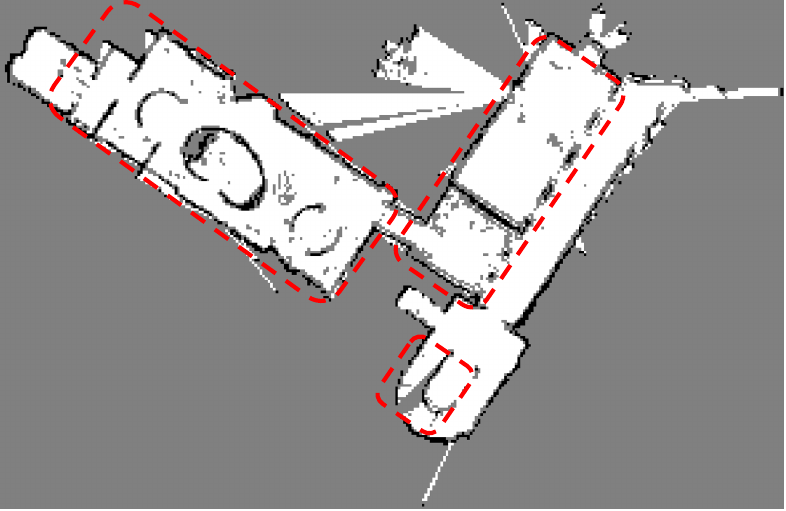}
    \caption{SLAM Toolbox}
    \label{fig:1g_slam}
\end{subfigure}
\hfill
\begin{subfigure}[b]{0.1945\textwidth}
    \centering
    \includegraphics[width=\textwidth]{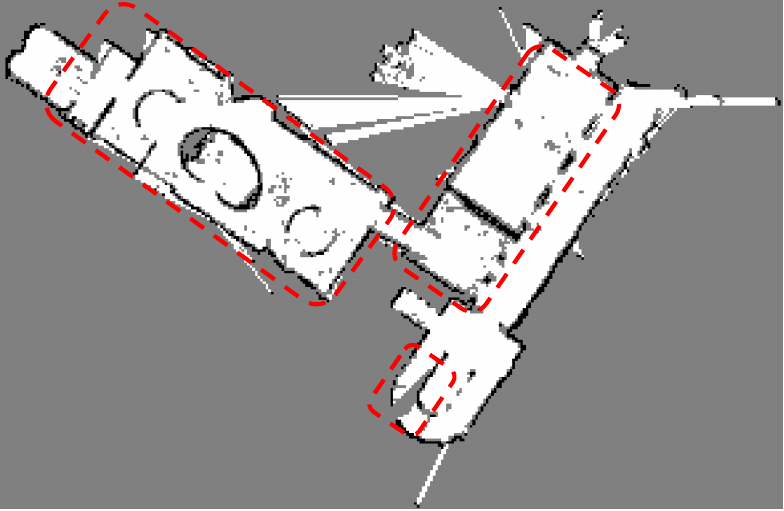}
    \caption{GBSJ}
    \label{fig:1g_gbsj}
\end{subfigure}
\hfill
\begin{subfigure}[b]{0.195\textwidth}
    \centering
    \includegraphics[width=\textwidth]{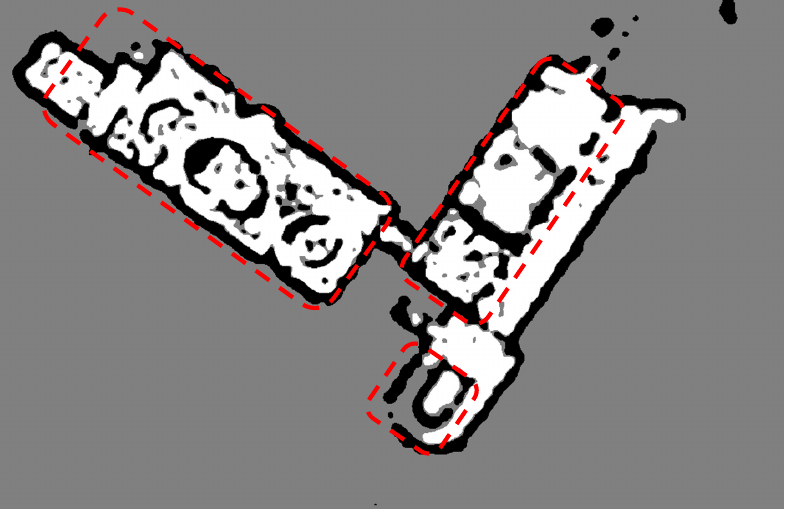}
    \caption{SBKM-Join}
    \label{fig:1g_sbkm}
\end{subfigure}
\hfill
\begin{subfigure}[b]{0.193\textwidth}
    \centering
    \includegraphics[width=\textwidth]{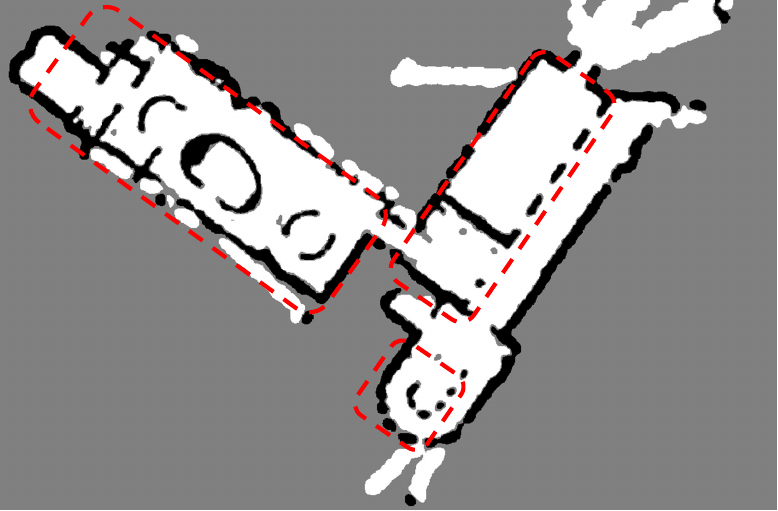}
    \caption{Ours (Full)}
    \label{fig:1g_ours}
\end{subfigure}

\caption{
Qualitative comparison of mapping results on Deutsches Museum
b0 1G. Red dashed boxes mark corresponding regions for comparison.
Black, white, and gray indicate occupied, free, and unknown
regions, respectively.
}
\label{fig:real_1g}
\vspace*{-1.5pt}
\end{figure*}

\begin{figure*}[t]
\centering

\begin{subfigure}[b]{0.30\textwidth}
    \centering
    \includegraphics[width=\textwidth]{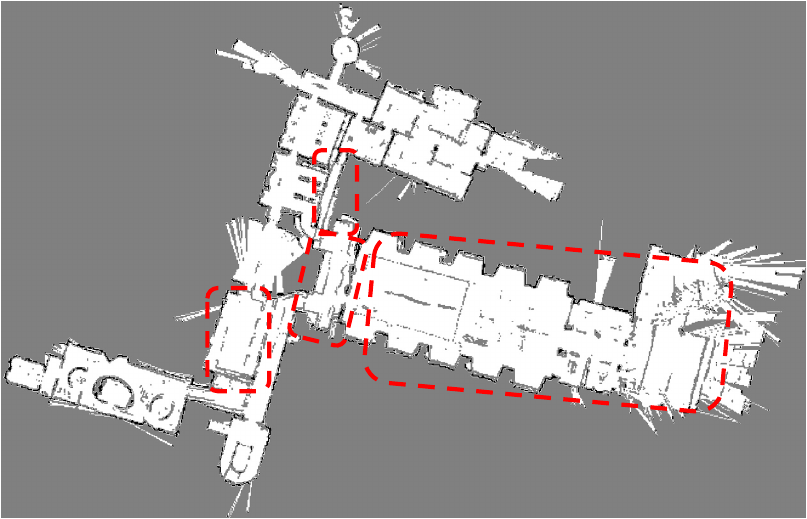}
    \caption{Cartographer}
    \label{fig:eg_carto}
\end{subfigure}
\hfill
\begin{subfigure}[b]{0.30\textwidth}
    \centering
    \includegraphics[width=\textwidth]{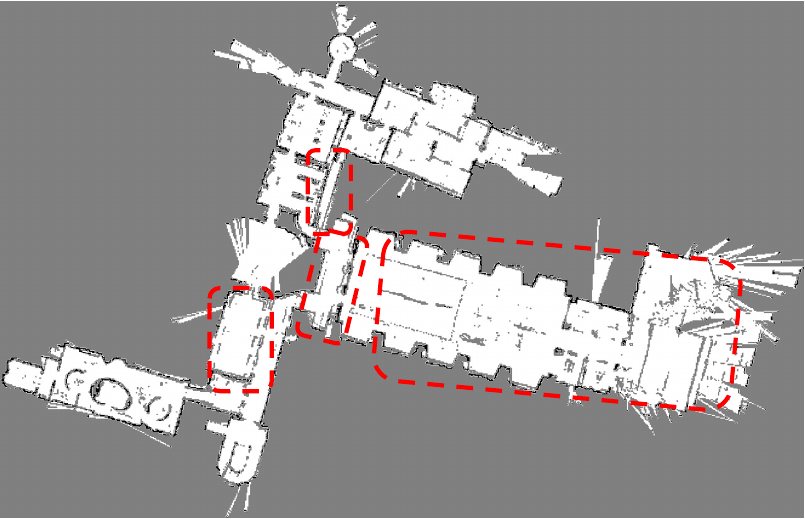}
    \caption{SLAM Toolbox}
    \label{fig:eg_slam}
\end{subfigure}
\hfill
\begin{subfigure}[b]{0.30\textwidth}
    \centering
    \includegraphics[width=\textwidth]{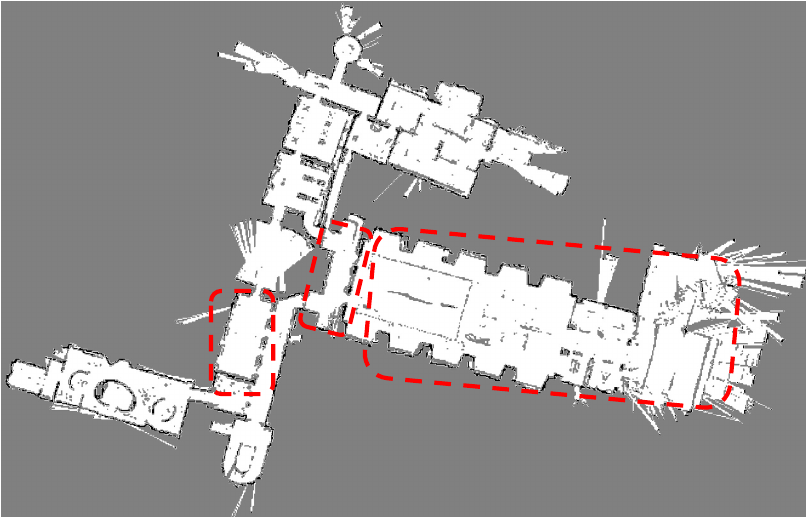}
    \caption{GBSJ}
    \label{fig:eg_gbsj}
\end{subfigure}

\par\medskip

\begin{subfigure}[b]{0.30\textwidth}
    \centering
    \includegraphics[width=\textwidth]{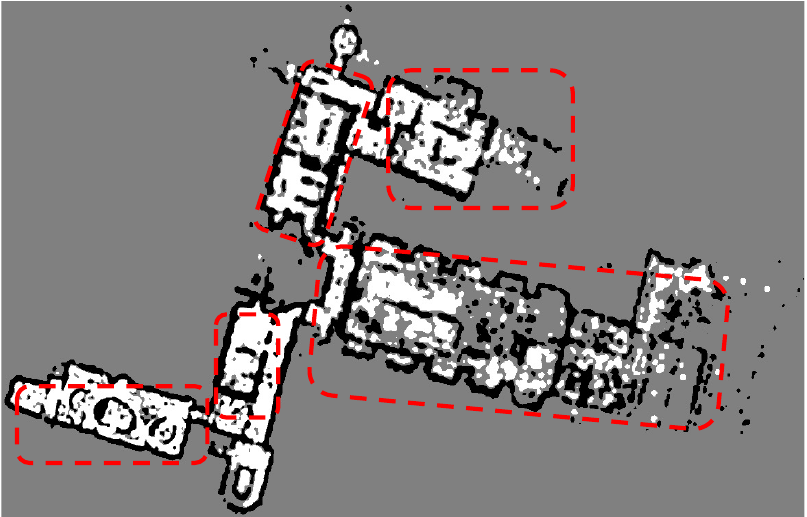}
    \caption{SBKM-Join}
    \label{fig:eg_sbkm}
\end{subfigure}
\hspace{0.06\textwidth}
\begin{subfigure}[b]{0.30\textwidth}
    \centering
    \includegraphics[width=\textwidth]{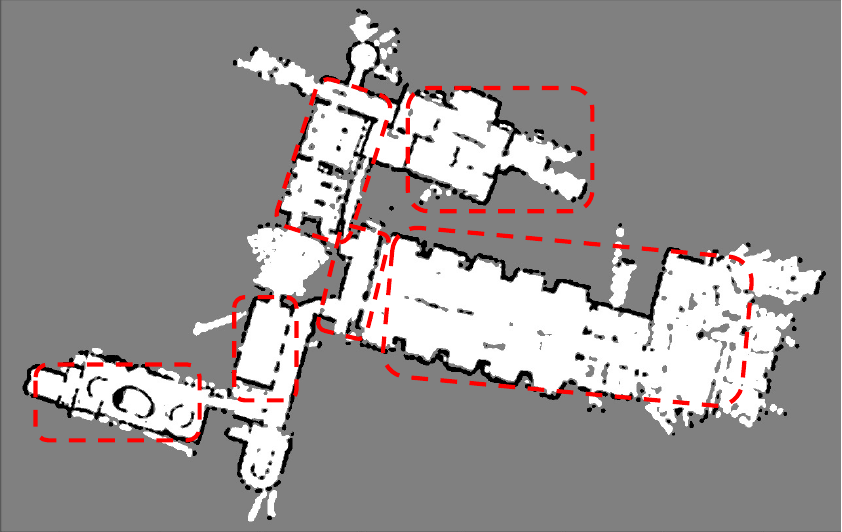}
    \caption{Ours (Full)}
    \label{fig:eg_ours}
\end{subfigure}

\caption{
Qualitative comparison of mapping results on Museum b0 EG.
Red dashed boxes mark corresponding regions for comparison.
Black, white, and gray indicate occupied, free, and unknown
regions, respectively.}
\label{fig:real_eg}
\vspace*{-1.5pt}
\end{figure*}
\subsubsection{Comparisons Using Real-world Datasets}
We compare our method with Cartographer, SLAM Toolbox, GBSJ,
and SBKM-Join on two real-world datasets~\cite{zhao2024occupancy}: Deutsches Museum
b0 1G  and Museum b0 EG. These datasets cover environments
of different scales, as summarized in Table~\ref{tab:datasets}.

Figs.~\ref{fig:real_1g} and~\ref{fig:real_eg} show that Ours (Full) preserves coherent
room and corridor structures across both environments.
Compared with Cartographer, SLAM Toolbox, and GBSJ, it
exhibits more continuous wall contours and fewer duplicated
or scattered boundary traces in the corresponding highlighted
regions. Compared with the continuous mapping-based joining baseline
SBKM-Join, Ours (Full) produces clearer room boundaries and
fewer scattered occupied predictions in free-space regions,
particularly on Museum b0 EG. Since GBSJ, SBKM-Join, and
our method use the same submap partitions and internal poses,
these results support the effectiveness of the proposed
mapping and joining formulation in large real-world environments.
\subsubsection{Representation, Efficiency, Convergence, and Scalability}
We further evaluate occupancy predictions, map storage,
and joining efficiency on the simulation datasets, and
map-construction scalability on Museum b0 EG.

\textbf{Occupancy Probability Comparison.}
Fig.~\ref{fig:comparison} also compares the continuous occupancy
probabilities predicted by SBKM-Join and Ours (Full), before
thresholding into occupied, free, and unknown regions.
Ours (Full) produces more spatially uniform predictions in
free-space regions and more clearly defined obstacle contours,
whereas SBKM-Join exhibits larger local variations.
Together with the quantitative results in
Tables~\ref{tab:pose_errors} and~\ref{tab:accuracy},
this comparison supports the effectiveness of the proposed
representation for submap joining and occupancy prediction.

\textbf{Efficiency and Convergence.} The joining time is measured after local submap construction, from
the start of submap joining to pose convergence. 
Table~\ref{tab:efficiency} shows that Ours (Full) requires less
per-submap storage than both GBSJ and SBKM-Join. Ours (Full) also requires fewer iterations and less joining time on
both datasets. Relative to SBKM-Join, joining is approximately
$3.13\times$ and $3.39\times$ faster in Simulations~1 and~2,
respectively. Both fewer iterations and lower average time per
iteration contribute to this reduction. Combined with the lower final
pose errors, these results show improved joining efficiency without
sacrificing estimation accuracy in the tested cases. Because SBKM-Join
uses the same joining backend, this comparison also indicates that
the local representation affects the computational cost and iteration count of
joining.
\begin{table}[t]
\centering
\caption{Map storage and submap joining efficiency.}
\label{tab:efficiency}
\setlength{\tabcolsep}{3pt}
\renewcommand{\arraystretch}{1.05}
\footnotesize

\begin{tabular}{ccccc}
\toprule
\makecell[c]{\textbf{Dataset}} &
\makecell[c]{\textbf{Method}} &
\makecell[c]{\textbf{Storage/Sub}\\\textbf{(MB)} $\downarrow$} &
\makecell[c]{\textbf{Iterations} $\downarrow$} &
\makecell[c]{\textbf{Joining Time}\\\textbf{(s)} $\downarrow$} \\
\midrule

\multirow{3}{*}{\textbf{Simulation 1}}
& GBSJ
& 8.0160
& \underline{15}
& 278.43 \\

& SBKM-Join
& \underline{0.0659}
& 21
& \underline{56.59} \\

& Ours (Full)
& \textbf{0.0250}
& \textbf{8}
& \textbf{18.10} \\

\midrule

\multirow{3}{*}{\textbf{Simulation 2}}
& GBSJ
& 8.0160
& \underline{12}
& 219.48 \\

& SBKM-Join
& \underline{0.0516}
& 16
& \underline{42.05} \\

& Ours (Full)
& \textbf{0.0189}
& \textbf{6}
& \textbf{12.40} \\

\bottomrule
\end{tabular}

\vspace{1mm}
\parbox{\columnwidth}{
\footnotesize
\raggedright
Storage denotes the average storage per submap: GBSJ stores
cell-wise log-odds and observation counts, while SBKM-Join
and our method store relevance vector locations, posterior
weight means, and parameters used for predictive uncertainty.
}
\end{table}

\textbf{Scalability.}
We evaluate the scalability of continuous occupancy map construction
on Museum b0 EG, the largest practical dataset. SBKM-Global and Ours-1 each construct a single global model
using all observations. Ours-4 and Ours-8 partition the same
observations into four and eight local submaps, respectively,
by uniformly dividing the input frames into consecutive
subsets. Construction time is measured from observation
input to completion of the global model or all local submaps, excluding
subsequent submap joining,  whose runtime is very small
relative to construction.

Table~\ref{tab:scalability} separates the effects of the mapping
formulation and submap decomposition. With one global model,
Ours-1 is approximately $2.46\times$ faster and requires
peak memory lower by a factor of $2.32$ than SBKM-Global,
supporting the computational benefit of the proposed
average-log-odds RVM formulation.
Within our formulation, increasing the number of submaps from one to
four and eight further reduces construction time and peak memory.
Ours-8 is $3.81\times$ faster and uses approximately one fifth of
the peak memory of Ours-1. This comparison directly supports
partitioning the observations into smaller learning problems to improve
map-construction scalability. The joining results in
Fig.~\ref{fig:real_eg} complement this construction-time evaluation by showing the applicability of joining in the same large environment.
\begin{table}[t]
\vspace*{3.5pt}
\centering
\caption{Scalability on the Museum b0 EG dataset.}
\label{tab:scalability}
\renewcommand{\arraystretch}{1}
\setlength{\tabcolsep}{3.5pt}
\footnotesize

\begin{tabular}{cccc}
\toprule
\multirow{2}{*}{\textbf{Method}} &
\textbf{No.} &
\textbf{Construction Time} &
\textbf{Peak Memory} \\
&
\textbf{Submaps} &
\textbf{(s)} $\downarrow$ &
\textbf{(MB)} $\downarrow$ \\
\midrule
SBKM-Global & 1 & 59895.48 & 43925.29 \\
Ours-1      & 1 & 24344.74 & 18920.94 \\
Ours-4           & 4 & \underline{11199.49} & \underline{7661.45} \\
Ours-8           & 8 & \textbf{6388.54} & \textbf{3799.15} \\
\bottomrule
\end{tabular}
\vspace*{-2.5pt}
\end{table}

\textbf{Assessment of Robustness to Initial Guess.}
To assess robustness to initialization errors, we perturb
the ground-truth poses of the non-reference submaps with
zero-mean uniformly distributed noise on both simulation
datasets. Three noise levels are considered: Level~1
(translation: $[-1,1]$~m, rotation: $[-0.15,0.15]$~rad),
Level~2 ($[-2,2]$~m, $[-0.3,0.3]$~rad), and
Level~3 ($[-3,3]$~m, $[-0.45,0.45]$~rad).
We perform 10 trials at each noise level for each dataset.
Pose errors are averaged over successful trials only.

Taking Simulation~1 as an example,
Table~\ref{tab:robustness_initialization} shows that our method
successfully converges in all 10 trials at Levels~1 and~2.
Even under the larger initialization perturbations of
Level~3, it converges in $70\%$ of the trials.
The average pose errors over successful trials remain small,
although they increase at Level~3.
These results demonstrate robustness to the tested
initialization errors. Similar robustness was observed
on Simulation~2.
\begin{table}[htbp]
\centering
\caption{Robustness to Initialization.} 
\label{tab:robustness_initialization}
\renewcommand{\arraystretch}{1}


\setlength{\tabcolsep}{2.35pt} 

\begin{tabular}{lccc} 
\toprule
\multicolumn{1}{c}{\textbf{Noise Level}} & 
\makecell[c]{\textbf{Convergence}\\\textbf{Percentage}} & 
\makecell[c]{\textbf{Average MAE}\\\textbf{Translation (m)}} & 
\makecell[c]{\textbf{Average MAE}\\\textbf{Rotation (rad)}} \\ 
\midrule
Level 1 (1 m, 0.15 rad) & 100\% & 0.0168 & 0.0010 \\
Level 2 (2 m, 0.3 rad)   & 100\% & 0.0173 & 0.0011 \\
Level 3 (3 m, 0.45 rad) & 70\%  & 0.0284 & 0.0025 \\ 
\bottomrule
\end{tabular}
\vspace*{-2.5pt}
\end{table}

\section{Conclusion} 
We presented a framework for large-scale continuous occupancy
mapping via variance-weighted submap joining. Local RVM
submaps use cell-wise average log-odds and observation counts
to enable Gaussian-conjugate inference. Submap joining is
formulated as a joint pose--field optimization problem and
solved through an equivalent reduced pose-only GN system,
with uncertainty weights held fixed within each GN iteration.
After the optimized submap poses are obtained, the global
field is recovered through closed-form
variance-weighted fusion and converted to occupancy
probabilities without training an additional global model.
Experiments show improved pose and occupancy-map accuracy
over the evaluated baselines in simulations, reduced map
storage and joining time, and improved map-construction
scalability on a large real-world dataset.

A practical limitation is that local submaps require
sufficient observations and spatial coverage to provide
reliable predictions for joining. Future work will
investigate adaptive submap partitioning based on
information content and extend the framework to 3D.

\addtolength{\textheight}{-12cm}   

\end{document}